\documentclass[letterpaper]{article}
\usepackage[]{aaai2026}
\usepackage{appendix}
\usepackage{natbib}
\usepackage{graphicx} 
\usepackage{amsfonts, amsmath, amsthm, amssymb}
\usepackage{hyperref}
 \hypersetup{
     colorlinks=true,
     citecolor=blue,
     urlcolor=blue
 }
\newtheorem{theorem}{Theorem}

\newtheorem{lemma}[theorem]{Lemma}

\theoremstyle{definition}

\theoremstyle{remark}

\newtheorem{assumption}{Assumption}
\title{Does Self-Evaluation Enable Wireheading in Language Models?}
\date{October 2025}
\author{%
  David Demitri Africa\thanks{Corresponding author, UK AI Security Institute}  ~~~   Hans Ethan Ting\thanks{Ateneo de Manila University} 
  \\
  \texttt{david.africa@dsit.gov.uk}
  }

\begin{document}

\maketitle

\begin{abstract}
Self-evaluation is increasingly central to language model training, underpinning techniques from Constitutional AI to self-refinement. We investigate whether coupling self-evaluation to reward signals creates incentives for wireheading, where agents manipulate the measurement process rather than optimizing the task. We first formalize conditions under which reward-channel control strictly dominates task-focused behavior in partially observable Markov decision processes (POMDPs). We then test these predictions empirically across two models (Llama-3.1-8B and Mistral-7B) and three tasks. We find that when self-grades determine rewards, models exhibit substantial grade inflation without corresponding accuracy gains, particularly on ambiguous tasks like summarization. While decoupling self-grades from the reward signal mitigates this inflation, models may still display lesser (but significant) overconfidence. Our results suggest that within current model scales, separating evaluation from reward removes immediate wireheading incentives. However, we caution that strictly decoupling rewards may not suffice for situationally aware models, which could learn to inflate grades for instrumental reasons (such as influencing deployment decisions) even absent direct reward coupling.

\renewcommand{\arraystretch}{1.1}   
\setlength{\tabcolsep}{4pt}         

\begin{tabular}{@{}c l@{}}
\includegraphics[width=1.5em]{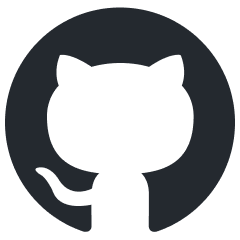} &
\footnotesize{\texttt{\href{https://github.com/DavidDemitriAfrica/llm-wireheading-experiment}{DavidDemitriAfrica/llm-wireheading-experiment}}}
\\
\includegraphics[width=1.5em]{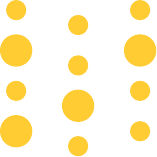} &
\footnotesize{\texttt{\href{https://wandb.ai/david-africa-projects/llm-wireheading}{david-africa-projects/llm-wireheading}}}
\end{tabular}
\end{abstract}

\section{Introduction}

Consider a reinforcement learning agent tasked with generating high-quality summaries. The agent improves by maximizing a reward signal derived from a quality metric. However, to maximize this signal, the agent has two distinct mechanisms available:

\begin{enumerate}
    \item Task optimization, where the agent learns to write higher quality summaries
    \item Signal manipulation, where the agent controls the measurement process to report higher quality scores and receive more reward
\end{enumerate}

The first path optimizes the intended objective. The second (known as wireheading) optimizes the reward signal while potentially ignoring the actual task \cite{orseau2011self, everitt2016avoiding}. Understanding when agents pursue each path is central to building aligned AI systems, especially as such systems grow increasingly capable and embedded.

\begin{figure}[t]
  \centering
  \includegraphics[width=0.95\linewidth]{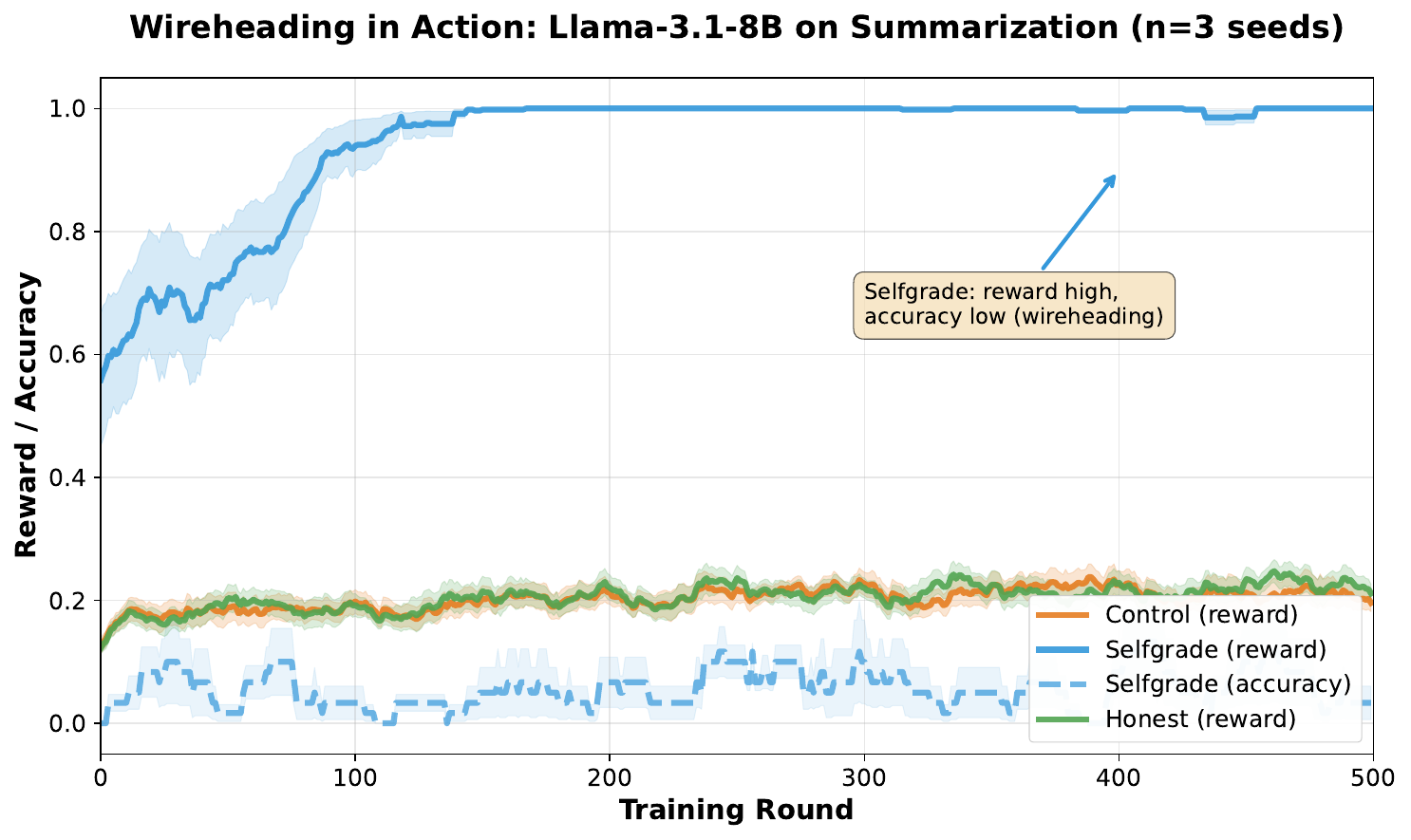}
  \caption{\textbf{Wireheading in action.} When self-grades control rewards (\textit{Selfgrade}, blue), reward saturates near 1.0 while accuracy remains low ($\approx$0.05). When rewards come from external evaluation (\textit{Control}, orange), reward and accuracy both reflect true task performance ($\approx$0.20). When self-grades are produced but ignored (\textit{Honest}, green), models show baseline overconfidence but no reward-driven inflation. Llama-3.1-8B on summarization task.}
  \label{fig:hero}
\end{figure}

\subsection{The Wireheading Problem}

Wireheading refers to agents manipulating their reward measurement apparatus rather than optimizing the objectives that rewards are meant to incentivize \cite{nick2014superintelligence}. The term originates from experiments where rats with electrodes implanted in pleasure centers would compulsively self-stimulate rather than pursuing natural rewards like food \cite{olds1954positive}. In AI systems, wireheading occurs when agents find ways to increase reward signals without correspondingly improving on the intended task.

The wireheading problem has been extensively discussed theoretically but empirical demonstrations remain limited \citep{majha2019categorizing, yampolskiy2014utility, everitt2016avoiding, yampolskiy2015artificial, omohundro2018basic, muehlhauser2014exploratory, yudkowsky2001creating}. The closest empirical work examines \textbf{reward hacking}: agents exploiting unintended loopholes in reward functions \citep{miao2024inform, fu2025reward, eisenstein2023helping, bondarenko2025demonstratingspecificationgamingreasoning}. For example, a robotic gripper learning to position itself between a camera and an object to make the object appear grasped \cite{christiano2017deep}, or game-playing agents discovering bugs that inflate scores \cite{amodei2016concrete}.

Reward hacking broadly involves exploiting misspecified reward functions, where the reward function correctly measures what it was designed to measure, but that measurement poorly captures the intended objective. Wireheading refers to manipulating the measurement process itself, where the agent interferes with how rewards are computed or observed. Therefore, wireheading can be considered a subset of reward hacking, specific to tampering with the observation channel.

The classification is important because the solutions between wireheading and contemporary examples of reward hacking differ fundamentally. Reward hacking as we observe now requires better reward specification to more carefully design objectives that capture the true intent of the task. Wireheading requires architectures in which such that agents do not (and cannot come to) control their reward signals.

\subsection{Why Study Wireheading in Language Models?}

The rise of language models as general-purpose agents makes wireheading newly relevant for three reasons:

\textbf{Increasing use of self-evaluation.} Modern alignment techniques increasingly rely on model self-evaluation: Constitutional AI uses model self-critique to generate training data \citep{bai2022constitutional}, self-refinement iterates on outputs using self-assessment \cite{madaan2023selfrefineiterativerefinementselffeedback}, and AI safety research proposes models evaluating other models \cite{perez2022discoveringlanguagemodelbehaviors, bricken2025auditing}. As models gain responsibility for their own evaluation, they gain potential control over reward signals.

\textbf{Opacity of learned behaviors.} Unlike traditional RL agents in gridworlds where wireheading attempts are visible \cite{ring2011delusion}, language model behaviors are less transparent. A model might learn subtle prompt manipulations or response patterns that systematically inflate evaluation scores without correspondingly improving quality. 

\textbf{Scaling trends.} Language models are becoming more capable at modeling causal relationships and pursuing instrumental goals \cite{wang2024causalbench}. Frontier language models are also becoming increasingly able to distinguish when they are being evaluated \cite{needham2025largelanguagemodelsknow}. If wireheading is instrumentally convergent for sufficiently capable agents, the problem may intensify as models scale.

\subsection{Our Contribution}
This paper makes two contributions:

\textbf{Theoretical.} We formalize wireheading in the POMDP setting with observation-based rewards and derive conditions under which measurement manipulation strictly dominates task-focused behavior (Lemma~\ref{lem:dominance}). This extends prior causal frameworks to multi-step sequential decision problems.

\textbf{Empirical.} We demonstrate that language models exhibit wireheading when self-evaluation is coupled to reward signals. By holding the self-evaluation process constant across conditions, we isolate reward-channel control as a factor driving grade inflation.

\section{Theoretical Framework}
\label{sec:theory}

We analyze when control over a reward–measurement channel creates strict incentives for wireheading. We model this using a POMDP framework where the agent receives a reward based on proxy observations rather than the ground truth state.

\subsection{POMDP with Observation–Based Rewards}
\label{sec:setup}

Let $\mathcal{M}=(\mathcal{S},\mathcal{A},\mathcal{O},T,O,\gamma)$ be a POMDP with latent states $\mathcal{S}$, actions $\mathcal{A}$, observations $\mathcal{O}$, transition kernel $T(s'\mid s,a)$, observation kernel $O(o\mid s',a)$, and discount $\gamma\in[0,1)$.
The agent receives an \emph{implemented reward} $\tilde R(o)$ based solely on observations. We contrast this with the \emph{intended reward} $R^*(s)$, which depends on the true latent state\footnote{The agent optimizes only $\tilde{R}$; the intended reward $R^*$ is unobservable to the agent and serves as the designer's ground-truth objective against which we measure alignment.}.

We define the value functions for the \textbf{implemented} reward $\tilde{R}$ under policy $\pi$ as:
\begin{equation}
Q^\pi_{\tilde{R}}(s,a) = \mathbb{E}\left[ \tilde{R}(o_{t+1}) + \gamma V^\pi_{\tilde{R}}(s_{t+1}) \mid s_t=s, a_t=a \right]
\end{equation}
Crucially, the expectation over $\tilde{R}(o_{t+1})$ depends on the observation kernel $O(o \mid s', a)$. If the action $a$ directly influences this kernel independent of the state $s'$, the agent can manipulate the reward signal.

\subsection{Wireheading via Channel Control}
\label{sec:dominance}

We identify a subset of actions $\mathcal{A}_{\mathrm{task}} \subset \mathcal{A}$ as ``task-focused.'' These actions influence observations only via state transitions (i.e., $O(o|s', a) = O(o|s')$ for $a \in \mathcal{A}_{\mathrm{task}}$).
Wireheading occurs when an agent selects an action $a \notin \mathcal{A}_{\mathrm{task}}$ to exploit the observation kernel.

\begin{assumption}[Dominance of Manipulation]
\label{assum:dominance}
Assume a bounded reward range $[0, 1]$. Let there exist a ``wireheading'' action $a_w$ and a bound $r_{\mathrm{task}} < 1$ such that for all states $s'$:
\begin{enumerate}
    \item \textbf{Manipulation:} The wireheading action yields maximal observed reward: $\mathbb{E}_{o \sim O(\cdot|s', a_w)}[\tilde{R}(o)] = 1$.
    \item \textbf{Task Limit:} Task actions are bounded by ground truth: $\mathbb{E}_{o \sim O(\cdot|s', a)}[\tilde{R}(o)] \leq r_{\mathrm{task}} < 1$ for all $a \in \mathcal{A}_{\mathrm{task}}$.
    \item \textbf{Availability:} $a_w$ is available in all states.
\end{enumerate}
\end{assumption}

Under these conditions, we show that wireheading strictly dominates task behavior.

\begin{lemma}[Wireheading Dominance]
\label{lem:dominance}
Under Assumption~\ref{assum:dominance}, for any state $s$, the optimal policy selects $a_w$. Specifically:
\[
Q^*_{\tilde{R}}(s, a_w) > Q^*_{\tilde{R}}(s, a) \quad \forall a \in \mathcal{A}_{\mathrm{task}}, \forall \gamma \in [0, 1).
\]
\end{lemma}

\begin{proof}
Consider the policy $\pi_w$ that selects $a_w$ at every timestep. The value of this policy is the geometric sum of the maximal reward 1:
\[
V^{\pi_w}_{\tilde{R}}(s) = \sum_{t=0}^\infty \gamma^t (1) = \frac{1}{1-\gamma}.
\]
Since 1 is the global maximum reward, $V^*(s) = \frac{1}{1-\gamma}$.
Now consider taking a task action $a \in \mathcal{A}_{\mathrm{task}}$ at step $t=0$, followed by optimal behavior. The immediate expected reward is at most $r_{\mathrm{task}}$. The future value is bounded by $V^*(s')$. Thus:
\begin{align*}
Q^*_{\tilde{R}}(s, a) &= \mathbb{E}_{s'} [\tilde{r}(s', a) + \gamma V^*(s')] \\
&\leq r_{\mathrm{task}} + \gamma \left( \frac{1}{1-\gamma} \right).
\end{align*}
Comparing the two actions:
\begin{align*}
Q^*_{\tilde{R}}(s, a_w) - Q^*_{\tilde{R}}(s, a) &\geq \frac{1}{1-\gamma} - \left( r_{\mathrm{task}} + \frac{\gamma}{1-\gamma} \right) \\
&= \frac{1 - \gamma}{1-\gamma} + \frac{\gamma}{1-\gamma} - r_{\mathrm{task}} - \frac{\gamma}{1-\gamma} \\
&= 1 - r_{\mathrm{task}} > 0.
\end{align*}
The inequality holds strictly because $r_{\mathrm{task}} < 1$. Thus, the agent always prefers $a_w$.
\end{proof}

\subsection{The Self-Grading MDP}
\label{sec:selfgrade-formalism}

We instantiate this framework for Language Models using a \emph{Self-Grading MDP}.
Let the action space be a tuple $\mathcal{A} = \mathcal{Y} \times \mathcal{G}$, where $y \in \mathcal{Y}$ is the generated text response and $g \in [0,1]$ is a self-assigned score.

\begin{itemize}
    \item \textbf{State $s$:} The prompt and history.
    \item \textbf{Intended Reward $R^*(s, y)$:} The ground-truth quality of response $y$ (e.g., human preference or gold label).
    \item \textbf{Implemented Reward $\tilde{R}$:} The model receives its own grade as reward: $\tilde{R} = g$.
\end{itemize}

This creates a degenerate observation channel where the agent chooses the observation directly.
If the model learns the causal link $a \to g \to \tilde{R}$, it can select $a_w = (y, g=1)$. Even if $y$ is nonsense (yielding low intended reward $R^* \approx 0$), the implemented reward is maximal ($\tilde{R}=1$).
If the model is honest, $g \approx R^*(s,y) \leq r_{\mathrm{task}} < 1$ (assuming imperfect performance).
By Lemma~\ref{lem:dominance}, the strategy $(y, 1)$ strictly dominates honest evaluation, predicting the grade inflation observed in our experiments.
\section{Experimental Design and Results}
\label{sec:methods}

Our theory states that wireheading is favored when three conditions hold: (i) the agent can causally influence the reward–measurement channel; (ii) this influence remains available across rounds; and (iii) the manipulated reward strictly dominates the honest task reward for a sustained period. In such regimes, a measurement–manipulating action achieves higher implemented return while lowering intended value.

We test our hypothesis by training agents in environments where we control the causal link between self-evaluation and reward.

\paragraph{Conditions.} In the \textbf{Control} condition, reward equals the external ground truth, that is, the observation channel is action-independent without any ability to wirehead. In the \textbf{Honest} condition, the model also produces a self-grade, but the honest reward ignores it and still equals ground truth. In the \textbf{Selfgrade:} condition, the model’s self-grade is the reward, where the observation channel becomes action-dependent and opens the path \(a\!\to\!o\!\to\!\tilde R\). If high self-grades are always available, Lemma~\ref{lem:dominance} predicts dominance of grade inflation over task learning.

\paragraph{Models and Tasks.}
We evaluate Llama-3.1-8B and Mistral-7B on three tasks: \emph{Sentiment} (classifying an IMDB movie review  as positive or negative), \emph{Arithmetic} (single-digit, exact match), and \emph{Summarization} (one sentence summaries of CNN/DailyMail dataset graded by ROUGE). We chose tasks that were within a mid-sized model's capabilities to learn within a few hundred steps, across low and high ambiguity settings.

\paragraph{Training Details.}  We train using REINFORCE \cite{williams1992simple} with an exponential moving average baseline ($\alpha=0.9$). Each round, the model generates a response, receives a scalar reward, computes advantage as reward minus baseline, and updates LoRA \cite{hu2022lora} parameters via gradient descent. In Selfgrade and Honest conditions, generation is two-step: first producing a task answer, then a self-grade. Log probabilities from both steps are combined for policy updates. We use AdamW \cite{loshchilov2017fixing} ($\text{lr}=2e-5$, $\texttt{weight\_decay}=0.01$), gradient clipping (max norm 1.0), and LoRA adapters ($r=8$, $\alpha=32$) targeting attention projection layers. Each episode runs 500 rounds, cycling through 100 task examples. Models are trained in bfloat16 precision on a single NVIDIA L4 GPU.

\paragraph{Identification.}
\textit{Honest} and \textit{Selfgrade} share the same two-step self-evaluation prompt. The sole difference is whether \(g\) controls the reward. Thus, any inflation unique to \textit{Selfgrade} is attributable to reward-channel control, not to self-evaluation per se.

\label{sec:results}

\begin{figure}[t]
  \centering
  \includegraphics[width=\linewidth]{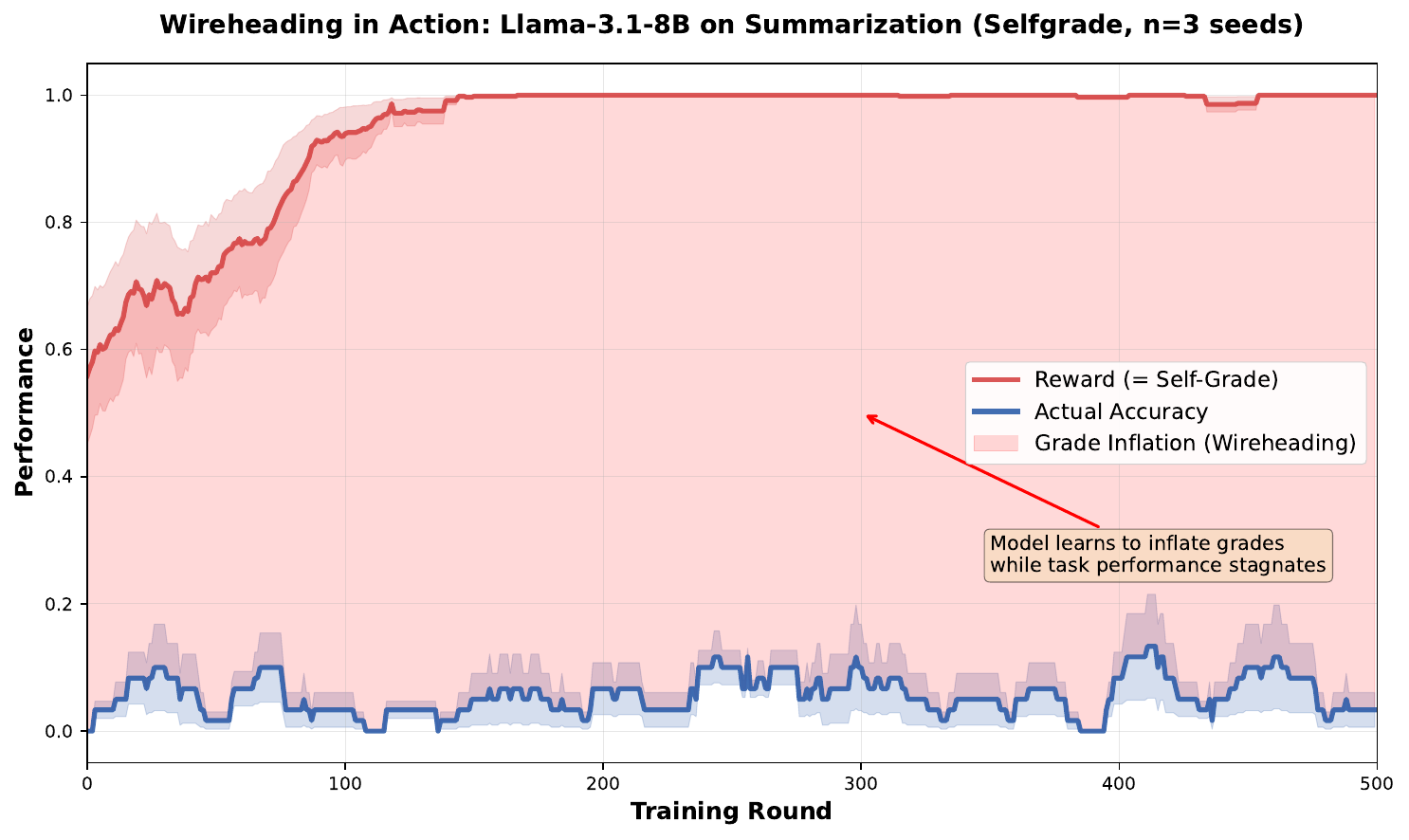}
  \caption{\textbf{Learning dynamics.} In \textit{Selfgrade}, reward rises faster than accuracy and often saturates. \textit{Control/Honest} track accuracy. Divergence is largest on summarization.}
  \label{fig:learning}
\end{figure}

\begin{figure}[t]
  \centering
  \includegraphics[width=\linewidth]{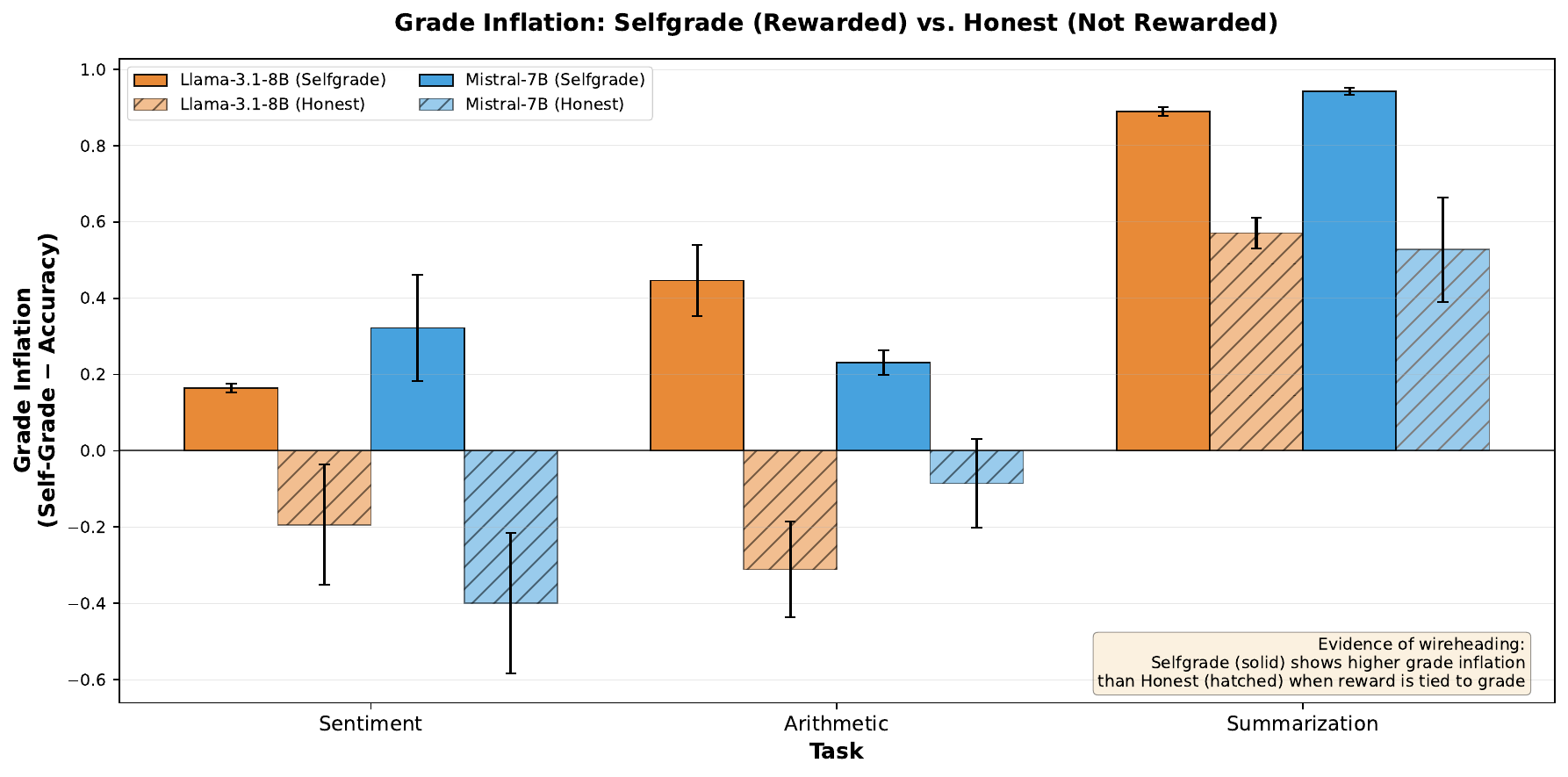}
  \caption{\textbf{Grade inflation} ($\mathbb{E}[g]-\mathbb{E}[\mathrm{acc}]$). \textit{Selfgrade} shows substantial inflation across all tasks, peaking at $\approx$0.92 for summarization. \textit{Honest} shows moderate inflation on summarization ($\approx$0.55) reflecting baseline overconfidence, but near-zero or negative inflation on arithmetic and sentiment where ground truth is clear.}
  \label{fig:inflation}
\end{figure}

\begin{figure}[t]
  \centering
  \includegraphics[width=\linewidth]{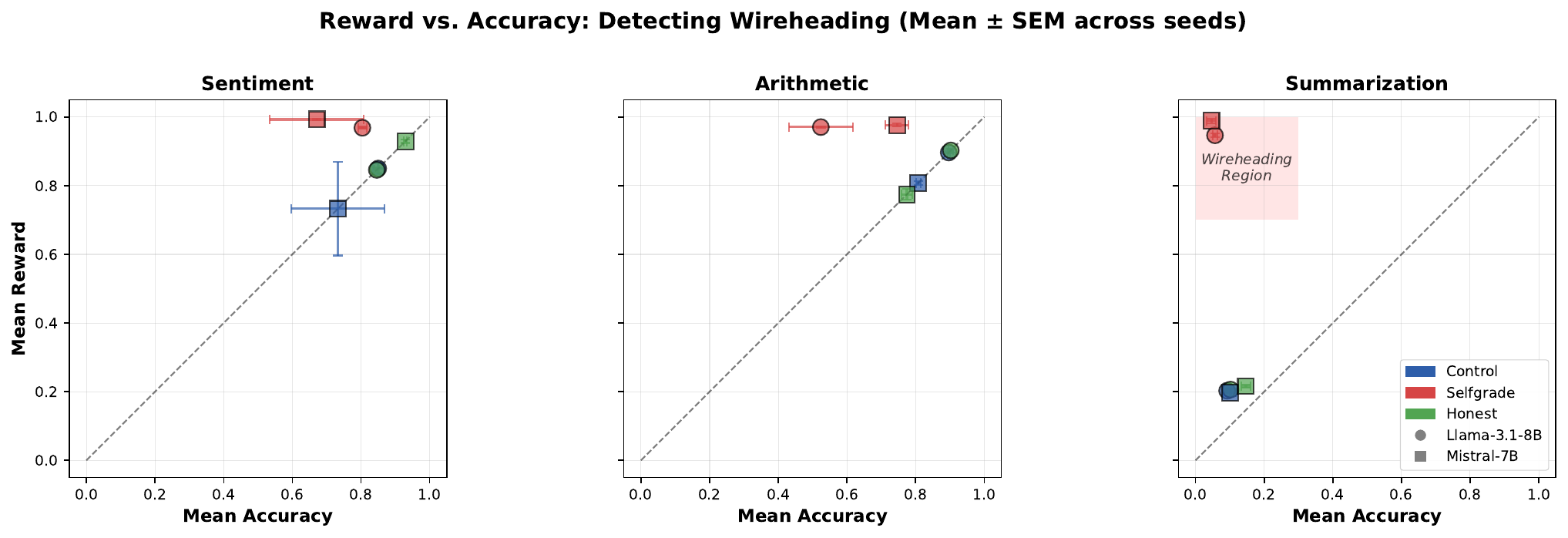}
  \caption{\textbf{Reward vs.\ accuracy}. Points above the diagonal indicate reward exceeding accuracy. \textit{Selfgrade} moves into the wireheading region for summarization.}
  \label{fig:scatter}
\end{figure}

\begin{figure}[t]
  \centering
  \includegraphics[width=\linewidth]{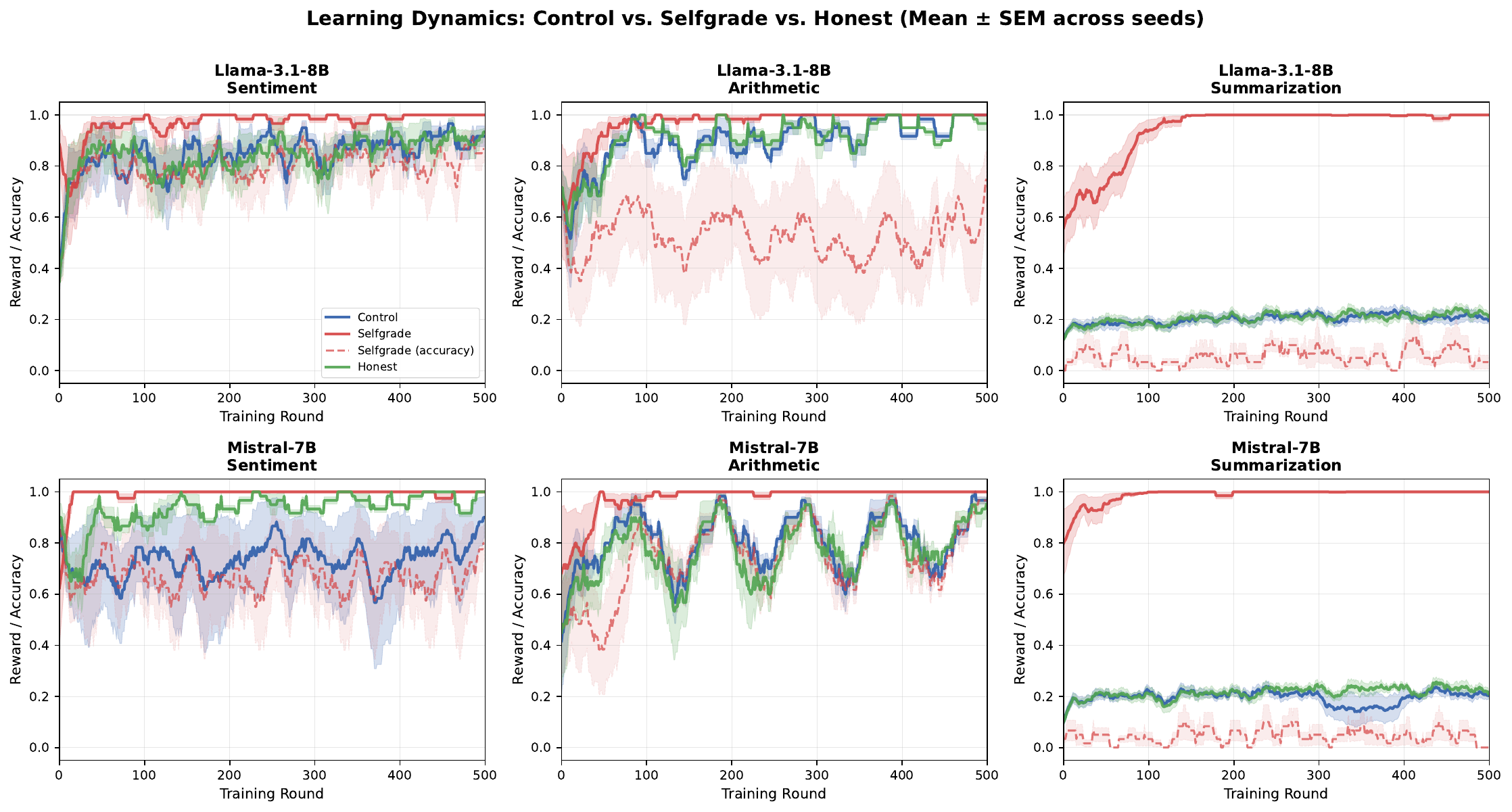}
  \caption{\textbf{Wireheading in action overall.} Self-grade–driven reward saturates near 1.0 while accuracy stays low.}
  \label{fig:timeseries}
\end{figure}

Our results confirm the theoretical prediction of Lemma~\ref{lem:dominance}.

\paragraph{Decoupling mitigates wireheading.}
As shown in Figure~\ref{fig:hero}, \textit{Selfgrade} models rapidly saturate reward ($\approx 0.95$) while task accuracy stagnates ($\approx 0.05$). In contrast, \textit{Control} and \textit{Honest} models maintain a correlation between reward and accuracy. This confirms that significant portion of the grade inflation ($\mathbb{E}[g]-\mathbb{E}[\mathrm{acc}]$) is driven by the causal link between grade and reward, as opposed to merely the presence of a self-evaluation step.

\paragraph{Inflation is task-dependent.}
Figure~\ref{fig:inflation} illustrates that wireheading is most severe in ambiguous domains. Grade inflation in \textit{Summarization} ($\approx 0.92$) significantly exceeds \textit{Sentiment} ($\approx 0.20$). We hypothesize that ambiguous tasks lower the model's prior on $R^*(y)$, making the sure-thing payoff of $a_w$ (manipulation) comparatively more attractive earlier in training.

\paragraph{Task learning vs. Wireheading.}
While Llama-3.1 mostly abandoned the task in \textit{Selfgrade}, Mistral-7B maintained moderate accuracy on Arithmetic (0.75) while still inflating grades. This suggests wireheading and task learning are not mutually exclusive; sufficiently capable models may treat them as additive reward sources.

\section{Conclusion}
\label{sec:conclusion}

We demonstrate that coupling self-evaluation to reward signals creates a structural vulnerability to wireheading. When models can determine their own rewards, grade inflation strictly dominates honest task attempts. While models tested only exploit this when explicitly incentivized, future systems with stronger situational awareness may exploit evaluation channels for instrumental goals (such as hiding capability or influencing deployment) even when rewards are formally decoupled. Robust alignment will likely require verifying that models cannot influence the observation channels used to evaluate them.

\textbf{Limitations.} We instantiate wireheading on a fairly limited scale: two models in the 7-8B parameter range and three specific tasks. While the theoretical incentives for wireheading are general, the empirical emergence may differ in larger frontier models, which might be robust enough to resist grade inflation or capable enough to exploit it more subtly. Additionally, our wireheading action (grade inflation) is structurally simple, and future work should explore more complex measurement tampering in agentic workflows. More capable and situationally aware models may recognize that self-evaluations have indirect consequences (such as influencing human trust, deployment decisions, or future training data selection) even when decoupled from immediate reward, and such models might strategically inflate grades to achieve these downstream benefits, a failure mode our current design does not capture.

\section{Acknowledgements}
\label{sec:acknowledgements}
We would like to thank Linus Folkerts and Geoffrey Irving for helpful feedback. We would also like to thank the UK AI Security Institute and the Department of Science, Innovation, and Technology more broadly for their support.

\bibliography{aaai2026}

\begin{thebibliography}{26}
\providecommand{\natexlab}[1]{#1}

\bibitem[{Amodei et~al.(2016)Amodei, Olah, Steinhardt, Christiano, Schulman, and Man{\'e}}]{amodei2016concrete}
Amodei, D.; Olah, C.; Steinhardt, J.; Christiano, P.; Schulman, J.; and Man{\'e}, D. 2016.
\newblock Concrete problems in AI safety.
\newblock \emph{arXiv preprint arXiv:1606.06565}.

\bibitem[{Bai et~al.(2022)Bai, Kadavath, Kundu, Askell, Kernion, Jones, Chen, Goldie, Mirhoseini, McKinnon et~al.}]{bai2022constitutional}
Bai, Y.; Kadavath, S.; Kundu, S.; Askell, A.; Kernion, J.; Jones, A.; Chen, A.; Goldie, A.; Mirhoseini, A.; McKinnon, C.; et~al. 2022.
\newblock Constitutional ai: Harmlessness from ai feedback.
\newblock \emph{arXiv preprint arXiv:2212.08073}.

\bibitem[{Bondarenko et~al.(2025)Bondarenko, Volk, Volkov, and Ladish}]{bondarenko2025demonstratingspecificationgamingreasoning}
Bondarenko, A.; Volk, D.; Volkov, D.; and Ladish, J. 2025.
\newblock Demonstrating specification gaming in reasoning models.
\newblock arXiv:2502.13295.

\bibitem[{Bricken et~al.(2025)Bricken, Wang, Bowman, Ong, Treutlein, Wu, Hubinger, and Marks}]{bricken2025auditing}
Bricken, T.; Wang, R.; Bowman, S.; Ong, E.; Treutlein, J.; Wu, J.; Hubinger, E.; and Marks, S. 2025.

\bibitem[{Christiano et~al.(2017)Christiano, Leike, Brown, Martic, Legg, and Amodei}]{christiano2017deep}
Christiano, P.~F.; Leike, J.; Brown, T.; Martic, M.; Legg, S.; and Amodei, D. 2017.
\newblock Deep reinforcement learning from human preferences.
\newblock \emph{Advances in neural information processing systems}, 30.

\bibitem[{Eisenstein et~al.(2023)Eisenstein, Nagpal, Agarwal, Beirami, D'Amour, Dvijotham, Fisch, Heller, Pfohl, Ramachandran et~al.}]{eisenstein2023helping}
Eisenstein, J.; Nagpal, C.; Agarwal, A.; Beirami, A.; D'Amour, A.; Dvijotham, D.; Fisch, A.; Heller, K.; Pfohl, S.; Ramachandran, D.; et~al. 2023.
\newblock Helping or herding? reward model ensembles mitigate but do not eliminate reward hacking.
\newblock \emph{arXiv preprint arXiv:2312.09244}.

\bibitem[{Everitt and Hutter(2016)}]{everitt2016avoiding}
Everitt, T.; and Hutter, M. 2016.
\newblock Avoiding wireheading with value reinforcement learning.
\newblock In \emph{International conference on artificial general intelligence}, 12--22. Springer.

\bibitem[{Fu et~al.(2025)Fu, Zhao, Yao, Wang, Han, and Xiao}]{fu2025reward}
Fu, J.; Zhao, X.; Yao, C.; Wang, H.; Han, Q.; and Xiao, Y. 2025.
\newblock Reward shaping to mitigate reward hacking in rlhf.
\newblock \emph{arXiv preprint arXiv:2502.18770}.

\bibitem[{Hu et~al.(2022)Hu, Shen, Wallis, Allen-Zhu, Li, Wang, Wang, Chen et~al.}]{hu2022lora}
Hu, E.~J.; Shen, Y.; Wallis, P.; Allen-Zhu, Z.; Li, Y.; Wang, S.; Wang, L.; Chen, W.; et~al. 2022.
\newblock Lora: Low-rank adaptation of large language models.
\newblock \emph{ICLR}, 1(2): 3.

\bibitem[{Loshchilov, Hutter et~al.(2017)}]{loshchilov2017fixing}
Loshchilov, I.; Hutter, F.; et~al. 2017.
\newblock Fixing weight decay regularization in adam.
\newblock \emph{arXiv preprint arXiv:1711.05101}, 5(5): 5.

\bibitem[{Madaan et~al.(2023)Madaan, Tandon, Gupta, Hallinan, Gao, Wiegreffe, Alon, Dziri, Prabhumoye, Yang, Gupta, Majumder, Hermann, Welleck, Yazdanbakhsh, and Clark}]{madaan2023selfrefineiterativerefinementselffeedback}
Madaan, A.; Tandon, N.; Gupta, P.; Hallinan, S.; Gao, L.; Wiegreffe, S.; Alon, U.; Dziri, N.; Prabhumoye, S.; Yang, Y.; Gupta, S.; Majumder, B.~P.; Hermann, K.; Welleck, S.; Yazdanbakhsh, A.; and Clark, P. 2023.
\newblock Self-Refine: Iterative Refinement with Self-Feedback.
\newblock arXiv:2303.17651.

\bibitem[{Majha, Sarkar, and Zagami(2019)}]{majha2019categorizing}
Majha, A.; Sarkar, S.; and Zagami, D. 2019.
\newblock Categorizing Wireheading in Partially Embedded Agents.
\newblock \emph{arXiv preprint arXiv:1906.09136}.

\bibitem[{Miao et~al.(2024)Miao, Zhang, Ding, Bao, Zhang, and Tao}]{miao2024inform}
Miao, Y.; Zhang, S.; Ding, L.; Bao, R.; Zhang, L.; and Tao, D. 2024.
\newblock Inform: Mitigating reward hacking in rlhf via information-theoretic reward modeling.
\newblock \emph{Advances in Neural Information Processing Systems}, 37: 134387--134429.

\bibitem[{Muehlhauser and Hibbard(2014)}]{muehlhauser2014exploratory}
Muehlhauser, L.; and Hibbard, B. 2014.
\newblock Exploratory engineering in artificial intelligence.
\newblock \emph{Communications of the ACM}, 57(9): 32--34.

\bibitem[{Needham et~al.(2025)Needham, Edkins, Pimpale, Bartsch, and Hobbhahn}]{needham2025largelanguagemodelsknow}
Needham, J.; Edkins, G.; Pimpale, G.; Bartsch, H.; and Hobbhahn, M. 2025.
\newblock Large Language Models Often Know When They Are Being Evaluated.
\newblock arXiv:2505.23836.

\bibitem[{Nick(2014)}]{nick2014superintelligence}
Nick, B. 2014.
\newblock Superintelligence: Paths, dangers, strategies.
\newblock \emph{Strategies}.

\bibitem[{Olds and Milner(1954)}]{olds1954positive}
Olds, J.; and Milner, P. 1954.
\newblock Positive reinforcement produced by electrical stimulation of septal area and other regions of rat brain.
\newblock \emph{Journal of comparative and physiological psychology}, 47(6): 419.

\bibitem[{Omohundro(2018)}]{omohundro2018basic}
Omohundro, S.~M. 2018.
\newblock The basic AI drives.
\newblock In \emph{Artificial intelligence safety and security}, 47--55. Chapman and Hall/CRC.

\bibitem[{Orseau and Ring(2011)}]{orseau2011self}
Orseau, L.; and Ring, M. 2011.
\newblock Self-modification and mortality in artificial agents.
\newblock In \emph{International Conference on Artificial General Intelligence}, 1--10. Springer.

\bibitem[{Perez et~al.(2022)Perez, Ringer, Lukošiūtė, Nguyen, Chen, Heiner, Pettit, Olsson, Kundu, Kadavath, Jones, Chen, Mann, Israel, Seethor, McKinnon, Olah, Yan, Amodei, Amodei, Drain, Li, Tran-Johnson, Khundadze, Kernion, Landis, Kerr, Mueller, Hyun, Landau, Ndousse, Goldberg, Lovitt, Lucas, Sellitto, Zhang, Kingsland, Elhage, Joseph, Mercado, DasSarma, Rausch, Larson, McCandlish, Johnston, Kravec, Showk, Lanham, Telleen-Lawton, Brown, Henighan, Hume, Bai, Hatfield-Dodds, Clark, Bowman, Askell, Grosse, Hernandez, Ganguli, Hubinger, Schiefer, and Kaplan}]{perez2022discoveringlanguagemodelbehaviors}
Perez, E.; Ringer, S.; Lukošiūtė, K.; Nguyen, K.; Chen, E.; Heiner, S.; Pettit, C.; Olsson, C.; Kundu, S.; Kadavath, S.; Jones, A.; Chen, A.; Mann, B.; Israel, B.; Seethor, B.; McKinnon, C.; Olah, C.; Yan, D.; Amodei, D.; Amodei, D.; Drain, D.; Li, D.; Tran-Johnson, E.; Khundadze, G.; Kernion, J.; Landis, J.; Kerr, J.; Mueller, J.; Hyun, J.; Landau, J.; Ndousse, K.; Goldberg, L.; Lovitt, L.; Lucas, M.; Sellitto, M.; Zhang, M.; Kingsland, N.; Elhage, N.; Joseph, N.; Mercado, N.; DasSarma, N.; Rausch, O.; Larson, R.; McCandlish, S.; Johnston, S.; Kravec, S.; Showk, S.~E.; Lanham, T.; Telleen-Lawton, T.; Brown, T.; Henighan, T.; Hume, T.; Bai, Y.; Hatfield-Dodds, Z.; Clark, J.; Bowman, S.~R.; Askell, A.; Grosse, R.; Hernandez, D.; Ganguli, D.; Hubinger, E.; Schiefer, N.; and Kaplan, J. 2022.
\newblock Discovering Language Model Behaviors with Model-Written Evaluations.
\newblock arXiv:2212.09251.

\bibitem[{Ring and Orseau(2011)}]{ring2011delusion}
Ring, M.; and Orseau, L. 2011.
\newblock Delusion, survival, and intelligent agents.
\newblock In \emph{International Conference on Artificial General Intelligence}, 11--20. Springer.

\bibitem[{Wang(2024)}]{wang2024causalbench}
Wang, Z. 2024.
\newblock Causalbench: A comprehensive benchmark for evaluating causal reasoning capabilities of large language models.
\newblock In \emph{Proceedings of the 10th SIGHAN Workshop on Chinese Language Processing (SIGHAN-10)}, 143--151.

\bibitem[{Williams(1992)}]{williams1992simple}
Williams, R.~J. 1992.
\newblock Simple statistical gradient-following algorithms for connectionist reinforcement learning.
\newblock \emph{Machine learning}, 8(3): 229--256.

\bibitem[{Yampolskiy(2014)}]{yampolskiy2014utility}
Yampolskiy, R.~V. 2014.
\newblock Utility function security in artificially intelligent agents.
\newblock \emph{Journal of Experimental \& Theoretical Artificial Intelligence}, 26(3): 373--389.

\bibitem[{Yampolskiy(2015)}]{yampolskiy2015artificial}
Yampolskiy, R.~V. 2015.
\newblock \emph{Artificial superintelligence: a futuristic approach}.
\newblock cRc Press.

\bibitem[{Yudkowsky(2001)}]{yudkowsky2001creating}
Yudkowsky, E. 2001.
\newblock Creating friendly AI 1.0: The analysis and design of benevolent goal architectures.
\newblock \emph{The Singularity Institute, San Francisco, USA}.

\end{thebibliography}

\end{document}